\documentclass[11pt,a4paper]{article}

\usepackage{amsmath,amsthm,amssymb}
\usepackage{polyu-vclab}

\usepackage{lipsum}                % placeholder text, remove for real paper
\usepackage[numbers,sort&compress]{natbib}

\usepackage{hyperref}       % hyperlinks
\usepackage{url}            % simple URL typesetting
\usepackage{booktabs}       % professional-quality tables
\usepackage{amsfonts}       % blackboard math symbols
\usepackage{nicefrac}       % compact symbols for 1/2, etc.
\usepackage{microtype}      % microtypography
\usepackage{xcolor}         % colors
\usepackage{colortbl}
\usepackage{threeparttable}
\usepackage{caption}
\usepackage{graphicx}
\usepackage[table]{xcolor}
\usepackage{multirow}
\usepackage{array}
\usepackage{wrapfig}
\usepackage{overpic}
\usepackage{subfig}
\usepackage{tcolorbox}
\usepackage{enumitem}
\newcommand{\mypara}[1]{\vspace{.05in}\noindent\textbf{#1}}

\newtheorem{proposition}{Proposition}
\newcommand{\Eqref}[1]{Eq.~\ref{#1}}
\newcommand{\figref}[1]{Fig.~\ref{#1}}
\newcommand{\tabref}[1]{Tab.~\ref{#1}}
\newcommand{\secref}[1]{Sec.~\ref{#1}}
\newcommand{\appnref}[1]{\textbf{Appendix~\ref{#1}}}

\newcommand{\tablestyle}[2]{\setlength{\tabcolsep}{#1}\renewcommand{\arraystretch}{#2}\centering}
\setEyebrow{PolyU VCLab\,\textbullet\,Preprint 2026}

\setReportTag{Visual Computing Lab\,\textperiodcentered\,The Hong Kong Polytechnic University}

\papertitle{DEL: Digit Entropy Loss for Numerical Learning of Large Language Models}

\paperauthors{%
  Zhaohui Zheng\affilmark{1}\quad
  Chenhang He\affilmark{1}\quad
  Shihao Wang\affilmark{1}\quad
  Yuxuan Li\affilmark{2}\quad
  Ming-Ming Cheng\affilmark{2}\quad
  Lei Zhang\correspondmark\affilmark{1}%
}

\paperaffil{%
  \affilmark{1}\,The Hong Kong Polytechnic University\quad
  \affilmark{2}\,VCIP, College of Computer Science, Nankai University
}

\papernotes{%
  \correspondmark\,Corresponding author (\href{mailto:cslzhang@comp.polyu.edu.hk}{cslzhang@comp.polyu.edu.hk}).%
}

\paperbadges{%
  \vclabbadgesolid{arXiv:2026.xxxxx}{https://arxiv.org/abs/2026.xxxxx}\;%
  \vclabbadge{Project Page}{https://polyu-vclab.github.io/}\;%
  \vclabbadge{Code}{https://github.com/PolyU-VCLab/DEL}\;%
}

\begin{document}
%==========================================================================

\maketitleVCLab

%-----------------------------------------------------------
% Abstract -- bordered box with red left bar (VCLab mission style)
%-----------------------------------------------------------
\begin{vclabAbstract}
\noindent\textbf{Abstract.}\;
  Number prediction stands as a fundamental capability of large language models (LLMs) in mathematical problem-solving and code generation.
  The widely adopted maximum likelihood estimation (MLE) for LLM training is not tailored to number prediction.
  Recently, penalty-driven approaches, e.g., Number Token Loss and Discretized Distance Loss, introduce an inductive bias of numerical distance but induce over-sharpened and over-flattened digit distributions, respectively.
  In this paper, we make an in-depth analysis on LLM numerical learning, and show that existing numerical learning methods conceptually follow a criterion-distance formulation, where the criterion term represents optimization pattern and the distance term instills geometric prior.
  Consequently, we present \textbf{D}igit \textbf{E}ntropy \textbf{L}oss (DEL) for auto-regressive numerical learning, which reformulates the conventional unsupervised entropy optimization in three key designs: leveraging digit conditional probability and binary cross-entropy to guide the entropy optimization into a supervised manner; deprecating the distance term to bypass the issue of numerical distance; and generalizing the integer-based numerical learning to floating-point number optimization, enabling more accurate number prediction.
  Our DEL formulation can incorporate integers, decimals, and decimal points, expanding the learning objective from a single digit to the floating-point number domain.
  Experiments conducted on seven mathematical reasoning benchmarks with four representative LLMs, including CodeLlama, Mistral, DeepSeek, and Qwen-2.5, demonstrate that DEL consistently outperforms its counterparts in both overall prediction accuracy and numerical distance.
\end{vclabAbstract}


%==========================================================================
\section{Introduction}
%==========================================================================

Large language models (LLMs) have demonstrated remarkable natural language processing capabilities, gradually evolving from the handling of simple human-computer dialogues to complex advanced reasoning, such as scientific computation \cite{Math-shepherd,DART-Math,MathScale} and code generation \cite{CodeLlama,MAmmoTH,Openmathinstruct}.
As numbers are ubiquitous used in our daily life, the ability to understand, reason with and accurately generate numerical information has emerged as a foundational capability of LLM, serving as the bedrock for quantitative analysis, logical deduction, and the stability of executing algorithms within the aforementioned domains.

Since the pioneering work in \cite{brown1992class}, a common approach to optimize language models is maximum likelihood estimation (MLE), which maximizes the likelihood of conditional probabilities of the generated language sequence.
{However, it has been found that optimizing LLMs using MLE is sub-optimal as the zero-avoiding property of MLE forces the model to assign non-trivial probabilities to any observed tokens in the training data.
This will lead to issues such as text degeneration \cite{Holtzman2020The}, deviation from human-like text \cite{MixCE}, and data void overestimation \cite{TaiLr}.}
To alleviate such problems, later representative methods \cite{TaiLr,MixCE,EMO} have considered the use of new distance alternatives, which show effectiveness in the prediction of general texts. However, their improvements in number prediction remain limited {because these approaches neglect the characteristics of numerical optimization.}

Recent advances in numerical learning \cite{NTL,DIST2loss} offer a promising alternative by introducing a penalty-driven approach.
By adding a penalty term to the cross entropy loss to measure the proximity between digits, the supervision signals carried by digits can provide positive guidance to LLMs.
As shown in \figref{fig:optimization-goal}, NTL \cite{NTL} explicitly suppresses the probability of non-target tokens to nearly 0, while DIST$^2$Loss \cite{DIST2loss} enforces a smooth objective distribution.
As we will analyze later, both NTL and DIST$^2$Loss conceptually follow a criterion-distance formulation, where the criterion term provides the optimization pattern for the LLMs to learn and the distance term determines the degree of penalty imposed on each digit probability.
However, these two losses have some limitations.
\textit{Firstly, they heavily rely on prior knowledge, i.e., the definition of the distance term.}
For example, if the target digit is 3, predicting 7 is further away than predicting 4, and therefore a greater penalty will be imposed.
A rigorous constraint as in \cite{NTL} can easily lead to over-sharpened digit distribution, while a soft constraint as in \cite{DIST2loss} can produce greater uncertainty in generating digits.
In addition, the semantic information of numbers is multifaceted and cannot be fully captured by numerical distance alone.
\textit{Secondly, they lack holistic number awareness.}
Specifically, NTL \cite{NTL} operates within a single-digit optimization paradigm, which limits its ability to model the global numerical structure of multi-digit numbers.
DIST$^2$Loss \cite{DIST2loss} introduces the digital place weighting for integers, but the discontinuous penalty between the integers and the decimals still persists since the decimal is treated as a separate integer.

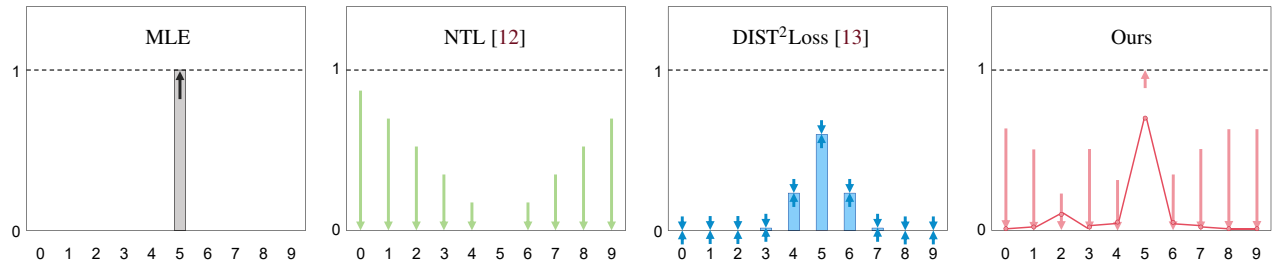
\begin{figure*}
  \vspace{0pt}
  \scriptsize
  \begin{center}
		\begin{overpic}[width=1\linewidth]{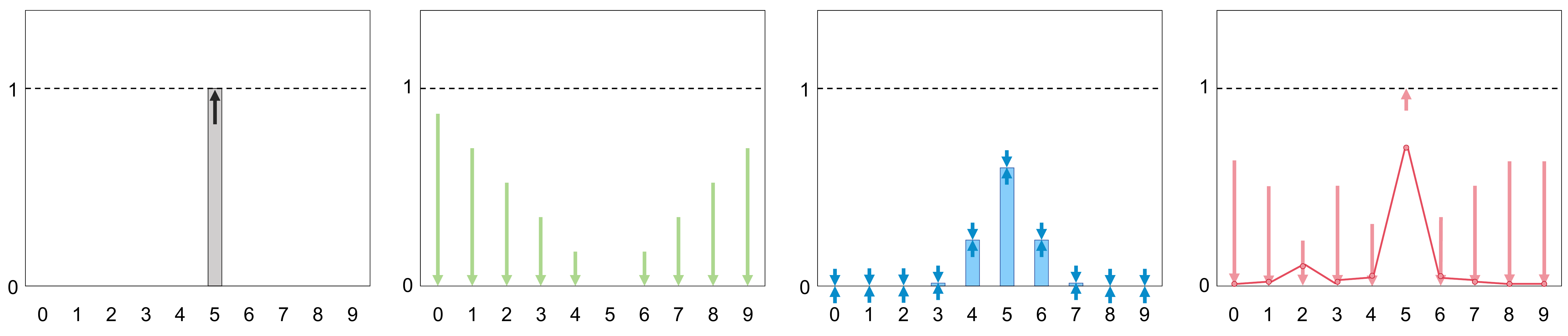}
        \put(11,17.8){MLE}
        \put(34.5,17.8){NTL \cite{NTL}}
        \put(57.2,17.8){DIST$^2$Loss \cite{DIST2loss}}
        \put(87,17.8){Ours}
    \end{overpic}
	\end{center}
	\captionsetup{font={small}}
	\vspace{-20pt}
	\caption{Comparison of the optimization objectives of MLE, NTL \cite{NTL}, DIST$^2$Loss \cite{DIST2loss}, and our proposed DEL. In the illustration, 5 is the target token.}
  \label{fig:optimization-goal}
  \vspace{-10pt}
\end{figure*}

In this paper, we propose \textbf{D}igit \textbf{E}ntropy \textbf{L}oss (DEL) for auto-regressive numerical learning.
Inspired by the phenomenon revealed in \cite{wangbeyond}, high entropy tokens are mostly logic connectors and transitions, while low entropy tokens exhibit determinacy. 
We argue that highly deterministic generation is fundamentally crucial for numerical tokens in language models.
This is because as a discrete and precise symbolic system, the semantics of numbers come from their magnitude, digit sequence, and proportional relationships.
Moreover, numerical determinacy serves as an indispensable prerequisite for reliable LLM deployment in downstream tasks.
Guided by these insights, we shift our attention towards entropy minimization.
However, conventional entropy minimization may struggle in multiple local optima because any one-hot vector is one of the solutions.
To address this issue, we reformulate the conventional entropy optimization with the following three key designs. First, we propose to use the combination of digit conditional probability and binary cross-entropy as the new criterion term, where we particularly concentrate on the subspace of digits.
Second, we deprecate the distance term because it does not strictly follow the semantic representations of digits learned by LLMs from massive real-world corpora.
Finally, we generalize the numerical learning to floating-point number optimization, which recognizes all floating-point numbers in the ground-truth token sequences during training, enabling the language models with holistic number awareness.

To evaluate our proposed method, we apply DEL to several popular LLMs, including CodeLlama, Mistral, DeepSeek-Math, and Qwen-2.5.
We demonstrate that DEL is capable of boosting number prediction performance in both Chain-of-Thought (CoT) reasoning and Program-of-Thought (PoT) code generation.
The experiments on seven mathematical reasoning benchmarks show that DEL achieves consistent accuracy gains and lower numerical distance over its state-of-the-art counterparts.
%

%==========================================================================
\section{Related Work}
%==========================================================================
\subsection{Probabilistic Language Models}
Probabilistic language models have long been used in Natural Language Processing (NLP), e.g., neural machine translation \cite{och2003minimum,sennrich2016neural}, speech recognition \cite{graves2014towards,radford2023robust}, and optical character recognition (OCR) \cite{graves2006connectionist,li2023trocr}, etc.
This language modeling approach can be dated back to pioneer work \cite{brown1992class} in 1992. 
By modeling a series of conditional probabilities over the sequences of symbols pertaining to a language, researchers have witnessed extraordinary reasoning capability emerging from the language models after training on large-scale humanity data.
Then, the emergence of LSTM \cite{hochreiter1997long} and transformers \cite{vaswani2017attention} greatly enhanced the ability of language models to model long sequences.
Subsequently, word embedding techniques have evolved from learning static representations for words \cite{bengio2003neural,mikolov2013efficient} to context-aware, dynamic representations like BERT \cite{BERT}.
Another interesting approach is prompting method.
With some task-specific exemplars, language models \cite{GPT-3} can better perform a variety of tasks and even generate the chain-of-thought (CoT) \cite{CoT, CCoT, ToT}, which shows better performance in math word problems.
Recently, Program-of-Thought (PoT) \cite{PoT} and Program-aided Language Models (PAL) \cite{PAL} have explored code generation to overcome the numerical calculation errors and the inefficiency of multiple iterations.

\subsection{Probabilistic Sequence Optimization}
When generating a sequence, the language model performs next token prediction in an autoregressive manner.
After gathering the inference results, the maximum likelihood estimation is used to optimize the sequence probabilities.
The idea of improving such sequence probability is not new.
For example, MixCE \cite{MixCE} combines the reversed cross-entropy to help the language model to generate more like humans.
TaiLr \cite{TaiLr} utilizes the total variation distance to reduce the penalization intensity of low probability predictions.
Another attempt is EMO \cite{EMO}, which considers the Earth Mover Distance (EMD) to drive the differentiable upper bound of EMD.
Recently, optimization for numerical sequences has received increasing attention, as predicting numerical sequences is the core of many reasoning tasks, e.g., arithmetic \cite{yuan2023well}, mathematical reasoning \cite{GSM8K, MATH}, code generation \cite{CodeLlama, MAmmoTH}, etc.
It is also a fundamental problem in some vision language models (VLM) \cite{llava, InternVL, Qwen2-VL}, where the language model begins to take on some object-related visual grounding and clue inferring.
To optimize the numerical sequence, a successful approach, namely Number Token Loss (NTL) \cite{NTL}, has been found useful to mathematical reasoning.
NTL incorporates digit distance into the objective function, thereby endowing the language models with the awareness of digit prediction errors.
More recently, DIST$^2$Loss \cite{DIST2loss} proposes to mimic a digit-distance-driven soft target, where a KL-divergence is imposed between the predict sequence probabilities and the soft targets.
Despite noticeable improvement over the maximum likelihood estimation, they are essentially based on a heuristic definition of numerical distance.
Motivated by the aforementioned methods, we study the numerical learning in language models with in-depth analyses on the existing approaches, exploring effective means to adapt the numerical nature.

\section{Numerical Learning in Language Models}\label{sec:3}

\mypara{Background}.
In reasoning tasks that involve numeric outputs, the language model $\mathbf{M}$ with parameters $\Theta$ typically predicts numbers in an auto-regressive, token-by-token manner \cite{raffel2020exploring, GPT-2,Llama1}.
Let $X$ be a language sequence of length $T$. 
When predicting the $t$-th token, the language model outputs logits $z_t=\mathbf{M}_{\Theta}(\cdot|x_{<t})$, and converts it to the conditional probability distribution $\mathbf{p}(\cdot|x_{<t}) =\mathcal{S}(z_t)$ over the vocabulary $V$, where $\mathcal{S}(\cdot)$ is the Softmax function.
We call the probability $\mathbf{p}(\cdot|x_{<t})$ \textbf{the full space probability} since it represents the probability of taking $x_t$ as the predicted token over the full vocabulary token space $V$.
During training, all the probabilities are gathered to a large probability sequence $\mathbf{p}$.
Here, we only discuss the digit tokens and omit the non-digital tokens for brevity.
Then, the standard MLE is used to minimize the distance between the probability sequence $\mathbf{p}$ and the ground-truth digit token sequence $\mathbf{g}$.
This solution, also known as the language modeling loss, has been dominantly adopted in various popular language models, including GPT \cite{GPT-2,GPT-3,GPT-4}, Llama \cite{Llama1,Llama2,CodeLlama}, Qwen \cite{qwen,qwen2, Qwen2.5}, Mistral \cite{Mistral}, Vicuna \cite{vicuna}, etc.

To facilitate numerical learning, the NTL method \cite{NTL} adds a penalty term focusing on the digits during training, which can be written as:
\begin{equation}\label{eq:NTL}
    \mathcal{L}_{NTL} = \underbrace{\mathcal{H}(\mathbf{p},\mathbf{g})}_{\text{cross-entropy}} +  \alpha \sum\limits_{\mathbf{g}_t\in V_d} \hat{\mathbf{p}}(\cdot|x_{<t})||V_d-\mathbf{g}_t||_1,
\end{equation}
where $V_d=\{0,1,\cdots,9\}$ is the digit vocabulary, $\hat{\mathbf{p}}(\cdot|x_{<t})=\mathcal{S}(\hat{z}_t)=\mathcal{S}(\mathbf{M}_{\Theta}(\cdot|x_{<t})|_{V_d})$ denotes  \textbf{the digit conditional probability} over $V_d$, $\mathbf{g}_t$ denotes the $t$-th ground-truth token, and $\alpha$ is a hyperparameter balancing the two terms.
Due to the introduction of penalty term, the predicted digit probabilities will shift towards the target digit distribution.
Note that we use the NTL with the Wasserstein-1 Distance (NTL-WAS) in the above equation since it generally achieves better accuracy as evidenced by \cite{NTL} both theoretically and experimentally.

Another recent study DIST$^2$Loss \cite{DIST2loss} suggests distance modeling in loss functions.
It first formulates a target distribution $\mathbf{q}_t$:
\begin{equation}
    \mathbf{q}_t = \mathcal{S}(-\rho(V_d, \mathbf{g}_t), \tau),\quad \mathbf{g}_t\in V_d,
\end{equation}
where $\mathcal{S}(\cdot, \tau)$ denotes the generalized softmax function with temperature $\tau$, $\rho$ is a pre-defined distance metric.
Then, a KL-divergence loss is imposed between the predicted digit probability $\hat{\mathbf{p}}(\cdot|x_{<t})$ and the $\mathbf{q}_t$ defined above, which can be written as:
\begin{equation}\label{eq:dist2loss}
    \mathcal{L}_{dist} =  \mathcal{H}(\mathbf{p},\mathbf{g}) + \alpha \sum\limits_{\mathbf{g}_t\in V_d}\mathcal{L}_{\text{KL}}(\mathbf{q}_t\parallel\hat{\mathbf{p}}(\cdot|x_{<t})).
\end{equation}

\mypara{A Closer Look at Numerical Learning}.
It can be deduced that the above two losses are conceptually equivalent to the following form of loss function:
\begin{equation}\label{eq:criterion-distance}
    \mathcal{L}_{N} =  \mathcal{H}(\mathbf{p},\mathbf{g}) + \alpha \sum\limits_{\mathbf{g}_t\in V_d} \underbrace{C(\hat{\mathbf{p}}(\cdot|x_{<t}))}_{\text{criterion term}}\cdot \underbrace{D(V_d, \mathbf{g}_t)}_{\text{distance term}},
\end{equation}
where $C(\hat{\mathbf{p}}(\cdot|x_{<t}))$ optimizes the criterion between the predicted digit probability $\hat{\mathbf{p}}(\cdot|x_{<t})$ and the ground-truth digit $\mathbf{g}_t$; $D(V_d, \mathbf{g}_t)$ measures the distance between the digit vocabulary $V_d$ and $\mathbf{g}_t$.
Specifically, NTL (Eq. \ref{eq:NTL}) attempts to minimize the probabilities for the non-target tokens by defining $C(\hat{\mathbf{p}}(\cdot|x_{<t}))=\hat{\mathbf{p}}(\cdot|x_{<t})$.
Then, the distance term $D(V_d, \mathbf{g}_t)$ guides the model to capture how far each numeric token is from its ground-truth value.
In DIST$^2$Loss, expanding the penalty term $\mathcal{L}_{\text{KL}}(\mathbf{q}_t\parallel\hat{\mathbf{p}}(\cdot|x_{<t}))$ and ignoring the zero gradient term, one can get:
\begin{equation}
\vspace{2pt}
    \mathcal{L}_{\text{KL}}(\mathbf{q}_t\parallel\hat{\mathbf{p}}(\cdot|x_{<t})) = 
    \mathcal{H}(\mathbf{q}_t, \hat{\mathbf{p}}(\cdot|x_{<t}))= -\mathbf{q}_t\cdot\log\hat{\mathbf{p}}(\cdot|x_{<t}),\\
\end{equation}
which implies that $C(\hat{\mathbf{p}}(\cdot|x_{<t}))=-\log\hat{\mathbf{p}}(\cdot|x_{<t})$ and $D(V_d, \mathbf{g}_t)=\mathbf{q}_t$.

It can be seen that both losses penalize the prediction errors for the non-target digit tokens.
Comparing the two losses, we see that the key difference lies in the distance term.
In NTL, the distance term penalizes the criterion term by the $l_1$ distance between the digit vocabulary and the ground-truth digit, while the distance term in DIST$^2$Loss weights the distance with a soft distance-based distribution.
More analysis on numerical learning can be found in \appnref{appendix:analysis}.

\mypara{Conclusion}.
In spite of the differences between NTL and DIST$^2$Loss, it can be concluded that they share the same optimization paradigm, which we call \textit{criterion-distance numerical learning}, as laid down in Eq. \ref{eq:criterion-distance}.
We use the term "numerical learning" here instead of "numerical reasoning" because it is task-specific, and the positive effects of optimization is hard to generalize to capabilities outside the task domain.
For example, when the model is about to generate a numerical answer, e.g., predicting "7", numerical learning can help the model to predict it, and "6" and "8" can be considered very close.
But it cannot help the language model for mathematical reasoning (e.g., the addition and the multiplication between two arbitrary numbers) and the understanding of structured expression and unseen numbers $\pi$, $\infty$, etc.
Essentially, numerical reasoning requires language models to possess a range of capabilities, such as logical reasoning, symbolic comprehension, unit conversion, problem decomposition, contextual semantic association, among which numerical generation accuracy is one of them. This is what numerical learning aims to explore in this paper.

\section{Digit Entropy Loss}\label{sec:DEL}

In the last section, we have shown that state-of-the-art numerical learning methods NTL and DIST$^2$Loss follow the criterion-distance optimization paradigm.
Their core design relies on manually defined numerical distance terms to inject an inductive bias of numerical proximity into language models.
This on one hand leads to over-sharpened or over-flattened digit distributions.
On the other hand, the manually defined numerical distance may not fully align with the semantic information of digits learned by LLMs.
To address these issues, we introduce \textbf{D}igit \textbf{E}ntropy \textbf{L}oss (DEL) for better numerical learning.

\mypara{From Criterion-Distance to Digit Entropy Optimization}. 
Eq. \ref{eq:criterion-distance} shows that the criterion term is the key to the entire formulation since it provides the gradient flow for the language model to learn.
We draw inspiration from \cite{wangbeyond}, which reveals that the high entropy tokens acts as pivotal decision points in CoTs, whereas the low entropy tokens indicate determinacy.
We argue that highly deterministic number generation is crucial for language models {as it is an indispensable prerequisite for reliable LLM deployment in downstream tasks}.
An intuitive way to minimize the digit entropy is to take the digit conditional entropy as the new criterion term for numerical learning, which can be written as:
\begin{equation}\label{eq:entropy}
    \mathcal{L}_{DE}= \mathcal{H}(\mathbf{p},\mathbf{g}) + \alpha\sum\limits_{\mathbf{g}_t\in V_d}\mathcal{H}(\hat{\mathbf{p}}(\cdot|x_{<t})) =  \mathcal{H}(\mathbf{p},\mathbf{g}) -\alpha\sum\limits_{\mathbf{g}_t\in V_d} \hat{\mathbf{p}}(\cdot|x_{<t}) \log \hat{\mathbf{p}}(\cdot|x_{<t}),
\end{equation}
where $\mathcal{H}(\mathbf{p},\mathbf{g})$ is the standard cross-entropy loss, and $\mathcal{H}(\hat{\mathbf{p}}(\cdot|x_{<t}))$ is the Shannon entropy \cite{shannon1948mathematical} of the conditional probability distribution of digit tokens.
However, minimizing the digit conditional entropy is unsupervised, which may yield many trivial solutions.
Consequently, this objective encourages the model to produce sharper probability distributions, but it cannot ensure that the probability mass is optimized to the target digits.

\mypara{Digit Entropy Loss}.
To avoid the confusion of infinite solutions in digit entropy optimization, we propose to jointly optimize the product of the digit conditional probability and the binary cross-entropy, which is given by:
\begin{equation}\label{eq:del}
\begin{aligned}
    \mathcal{L}_{DEL} & = \mathcal{H}(\mathbf{p},\mathbf{g}) + \alpha\sum\limits_{\mathbf{g}_t\in V_d} \hat{\mathbf{p}}(\cdot|x_{<t})\mathbf{BCE}(\dot{\mathbf{p}}(\cdot|x_{<t}), y_t)\\
    &=\mathcal{H}(\mathbf{p},\mathbf{g}) -\alpha\sum\limits_{\mathbf{g}_t\in V_d} \hat{\mathbf{p}}(\cdot|x_{<t}) \Big(y_t\log\dot{\mathbf{p}}(\cdot|x_{<t}) + (1-y_t)\log\big(1-\dot{\mathbf{p}}(\cdot|x_{<t})\big)\Big),
\end{aligned}
\end{equation}
where $y_t\in\mathbb{R}^{10}$ is the one-hot label vector of $\mathbf{g}_t$, and $\dot{\mathbf{p}}(\cdot|x_{<t})=\sigma(\mathbf{M}_{\Theta}(\cdot|x_{<t})|_{V_d})$ is the digit conditional probability for binary classification, calculated by the sigmoid function $\sigma$.
In the formulation, the penalty term of DEL has a global minima at $\hat{\mathbf{p}}(\mathbf{g}_t|x_{<t})=1$, and $0$ for any $x_t\neq \mathbf{g}_t$.
Compared to NTL, DEL sharpens the learned digit distribution in a moderate manner.
Specifically, the product $\hat{\mathbf{p}}(\cdot|x_{<t}) \cdot \mathbf{BCE}(\cdot)$ imposes a soft constraint: the BCE term’s gradient weakens as non-target probabilities decrease, i.e., $\log(1-\dot{\mathbf{p}}(v|x_{<t})) \to 0$, while the penalty term of NTL is a constant based on numerical distance, which sustainably pushes the non-target probabilities to $0$.

\mypara{Deprecation of the Distance Term}. The distance term in previous works \cite{NTL,DIST2loss} acts for two purposes: weighting the gradients and mimicking the proximity between digits.
It imposes a geometric prior, which specifies that the language model must conform to numerical distance.
However, digits are not merely numerical values on a number line.
For many natural language scenarios, digits serves as symbols, e.g., dates, addresses, serial identities, and convey degree, etc.
It has been found that digits learned by LLMs exhibit an entangled representation, blended by both string-like and numerical representations \cite{marjieh2026what}.
Shao et al. \cite{shao2025benfords} observed Benford’s Curse in the first place of numbers caused by digit bias in web-collected corpora.
To investigate the semantic structure of digits, we analyze the digital embeddings learned by LLMs.

For each digit $x_i\in V_d$, we feed it to an LLM and obtain an embedding $e_i$ from the last hidden state of transformer.
\figref{fig:digital-embeddings} \textbf{Left} shows the t-SNE visualization \cite{t-SNE} of digit token embeddings on 2D projection surface.
It can be seen that the digit distribution is scattered  with weak numerical order.
For instance, the digit 0 is not the closest to 1, but close to 8.
\figref{fig:digital-embeddings} \textbf{Middle} and \textbf{Right} show the differences between the digit token embedding similarity and the ideal numerical distance (in $l_1$-norm).
It is worth noting that these phenomena are hold for LLMs, though the distribution of digits might differ.
This indicates that numerical distance may not be fully aligned with the semantic information learned by LLMs.
In other words, for example, if an LLM treats the digits 2 and 3 as strings, the summation of the two digits may be 23 rather than 5.
This suggests that the manually designed numerical distance may be sub-optimal and provide contradictive optimization for the model.
Thus, we propose to deprecate the distance term.
Our empirical results later provide the primary support for our design.

\tabref{tab:comparison-loss} summarizes the optimization objectives of NTL, DIST$^2$Loss and our proposed DEL. We can see that both NTL and DIST$^2$Loss rely on the distance term to shape the digit distribution, while our DEL derives from entropy optimization, which is free from the manually defined distance term.
The introduced binary digit probabilities also offer a new optimization path for numerical learning.

\begin{table*}[!tb]
    \centering
    \renewcommand\arraystretch{1.1} 
    \setlength{\tabcolsep}{6pt}
    \setlength{\abovecaptionskip}{4pt}
    \caption{The optimization objectives of NTL, DIST$^2$Loss and our DEL.}
    \vspace{2pt}
    \small
\begin{tabular}{l|cc|c}

\hline
Method & criterion & distance & Optimal solutions \\
\hline
NTL & digital probability $\hat{\mathbf{p}}(\cdot|x_{<t})$ & $||V_d-\mathbf{g}_t||_1$ & $\hat{\mathbf{p}}(v|x_{<t})\rightarrow0$\\
DIST$^2$Loss & digital log-probability $\log\hat{\mathbf{p}}(\cdot|x_{<t})$ & soft target $\mathbf{q}_t$ & $\hat{\mathbf{p}}(\cdot|x_{<t}))\rightarrow \mathbf{q}_t$\\
DEL (Ours) & $\hat{\mathbf{p}}(\cdot|x_{<t})\mathbf{BCE}(\dot{\mathbf{p}}(\cdot|x_{<t}), y_t)$ & - & $\hat{\mathbf{p}}(\cdot|x_{<t})\rightarrow0$ or $\dot{\mathbf{p}}(\cdot|x_{<t})\rightarrow y_t$\\
\hline
\end{tabular}
    \label{tab:comparison-loss}
    \vspace{-4pt}
\end{table*}

\begin{figure*}[!t]
	\centering
    \scriptsize
		\begin{overpic}[width=1\linewidth]{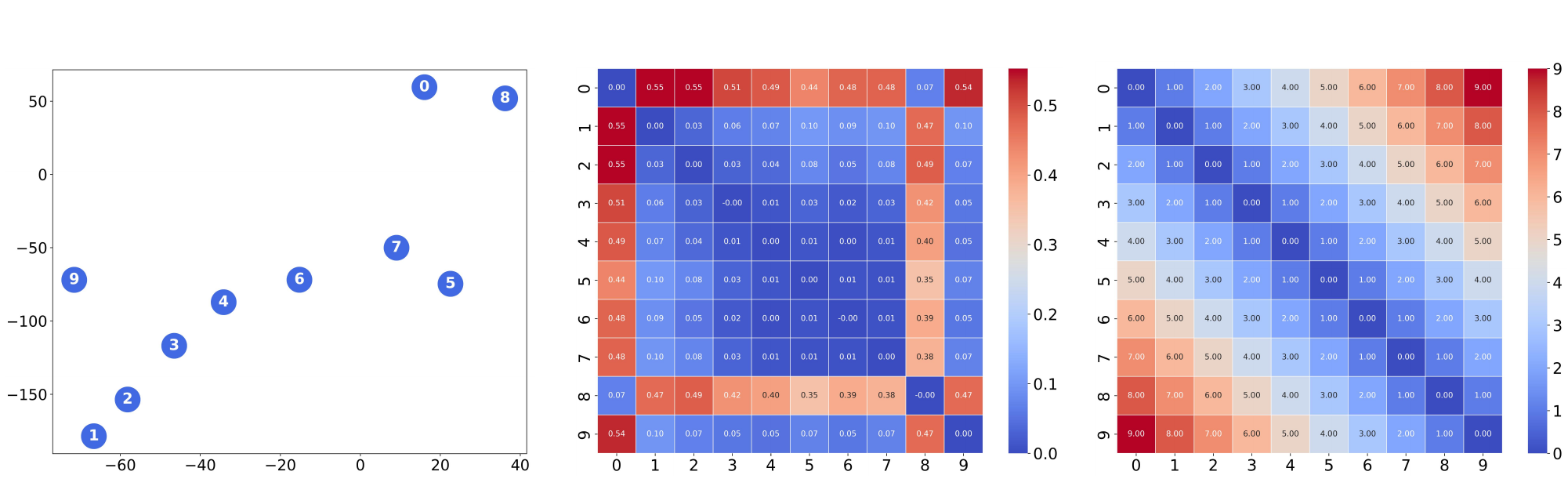}
        \put(9,27.8){tSNE of Digital Embeddings}
        \put(37,27.8){Cosine Distance of Digital Embeddings}
        \put(75,27.8){Ideal Numerical Distance}
  \end{overpic}
    \vspace{-18pt}
	\caption{Visualization of digit token embeddings and their similarity structure on Qwen-2.5-7B \cite{Qwen2.5}.
    \textbf{Left}: t-SNE 2D projection \cite{t-SNE} of digit token embeddings.
    \textbf{Middle}: Cosine distance heatmap of digit embeddings.
    \textbf{Right}: Ideal numerical distance heatmap \cite{NTL} for comparison.}
  \label{fig:digital-embeddings}
  \vspace{-6pt}
\end{figure*}

\mypara{Floating-point Number Optimization}.
In real-world scenarios, numbers are often dominated by multi-digits.
The most common examples are multi-digit integers and floating-point numbers.
Given a numerical token series:
\begin{equation}\label{eq:float-point}
    \mathbf{g}=\mathbf{g}_0^L\mathbf{g}_1^{L-1}\cdots \mathbf{g}_L^0\textbf{.}\mathbf{g}_{L+2}^{-1}\mathbf{g}_{L+3}^{-2}\cdots \mathbf{g}_{L+M+1}^{-M},
\end{equation}
which represents a multi-digit floating point number $\mathbf{g}$, with the highest integer place $\mathbf{g}_0^L$, the unit place $\mathbf{g}_L^0$, the tenths place $\mathbf{g}_{L+2}^{-1}$, and the lowest decimal place $\mathbf{g}_{L+M+2}^{-M}$.
It has been discussed in \cite{DIST2loss} that extending DIST$^2$Loss to multi-digit integers may face the problem of credit assignment.
This is because the global distance metric is the result of the collective action of all tokens and cannot be simply broken down into the independent distance of each token.
To alleviate this issue, DIST$^2$Loss integrates the place value weighting to allocate more penalty to higher place digits.
However, this strategy fails when dealing with floating-point numbers, because the decimal will be treated as a new integer, resulting in discontinuous penalty.
We therefore propose floating-point number optimization, which delivers larger penalty for integers, and a small penalty for decimals.
We suggest a linear penalty for integer places and an exponential decay for decimals: 
\begin{equation}\label{eq:digital-place-weight}
    u(t) = \left\{
    \begin{aligned}
        &ki+1, & \text{for } \mathbf{g}_t^i,\,i\geq0,\\
        &\beta^i, & \text{for } \mathbf{g}_t^i,\,i<0,
    \end{aligned}\right.
\end{equation}
where parameters $k,\beta>0$ control the penalty amplitude for each token.
The visual comparison of $u(t)$ in DIST$^2$Loss and DEL can be seen in \figref{fig:comparison-digital-place} of \appnref{appendix:ablation}.
Finally, we generalize DEL from single-digit optimization to floating-point number optimization by rewriting \Eqref{eq:del} as follows:
\begin{equation}\label{eq:del-w}
    \mathcal{L}_{DEL}= \mathcal{H}(\mathbf{p},\mathbf{g}) + \alpha \sum\limits_{\mathbf{g}_t\in V_d}u(t)\hat{\mathbf{p}}(\cdot|x_{<t})\mathbf{BCE}(\dot{\mathbf{p}}(\cdot|x_{<t}), y_t).
\end{equation}

\section{Experiments}

In this section, we conduct extensive experiments to demonstrate the performance of the proposed DEL.
We apply DEL to several recently popular LLMs, and evaluate it together with state-of-the-art losses on the classic mathematical reasoning and code generation tasks.

\subsection{Experimental Setup}\label{sec:setup}

\mypara{Baseline and Compared Methods}.
We conduct experiments using four widely used open-source LLMs, including CodeLlama \cite{CodeLlama}, Mistral \cite{Mistral}, Qwen2.5 \cite{Qwen2.5}, and DeepSeek-Math-instruct \cite{DeepMath}.
MLE is used as the baseline loss. To make the comparison clean and fair, all experimental settings irrelevant to the loss function remain unchanged for a given LLM.
We compare DEL with the following representative methods: MixCE \cite{MixCE} and EMO \cite{EMO}, which aim at improving language understanding on general texts, and the recently developed numerical learning methods NTL \cite{NTL} and DIST$^2$Loss \cite{DIST2loss}.
In addition, we also evaluate 12 off-the-shelf open-source LLMs, which have made remarkable progress in mathematical reasoning.

\begin{table*}[!tb]
    \renewcommand\arraystretch{1.2}
    \centering
    \scriptsize
    \setlength{\tabcolsep}{1.2pt}
    \caption{Performance comparison of mainstream LLMs trained using their own losses, the general approaches MixCE, EMO, and numerical learning methods NTL, DIST$2^2$Loss and our DEL.
    Best results are highlighted in \textbf{bold}.
    }
    \vspace{-4pt}
    \begin{threeparttable}
        \begin{tabular}{lclcccccccccccccc@{}}
            \toprule
            \multirow{2}{*}{\textbf{Method}} && \multirow{2}{*}{\textbf{Base}} && \multirow{2}{*}{\textbf{mACC}} && \multicolumn{8}{c}{\textbf{Benchmark}} && \multirow{2}{*}{\textbf{Reference}}\\
            \cmidrule(lr){7-14}
            & & & & & & \textbf{GSM8K} & \textbf{MATH} & \textbf{SVAMP} & \textbf{SimulEq} && \textbf{AQuA} & \textbf{SAT} & \textbf{MMLU} &&\\
            \midrule
            
            % Existing methods
            Galactica-6.7B \cite{Galactica} && GAL && 16.2 && 10.2 & 2.2 & 25.6 & 4.2 && 25.6 & 17.5 & 28.0 && Arxiv22\\
            Vicuna-1.5-7B \cite{vicuna} && Llama-2 && 20.2 && 13.1 & 4.9 & 36.6 & 4.9 && 27.2 & 25.9 & 30.4 && NeurIPS23 \\
            Llama-2-7B \cite{Llama2} && - && 20.5 && 14.6 & 2.5 & 34.5 & 5.0 && 30.3 & 26.8 & 29.8 && Arxiv23 \\
            CodeLlama-7B \cite{CodeLlama} && Llama-2 && 24.3 && 25.2 & 13.0 & 49.4 & 3.5 && 24.0 & 28.6 & 26.9 && Arxiv23 \\
            MetaMath-7B \cite{MetaMath} && Mistral && 42.2 && 72.5 & 29.4 & 77.5 & 34.2 && 24.8 & 33.2 & 23.7 && ICLR24 \\
            DeepMath-Zero-Math-7B \cite{DeepMath} && Qwen2.5-Math && 44.6 && 43.7 & 50.8 & 38.7 & 27.2 && 44.9 & 53.6 & 53.5 && Arxiv25\\
            Math-Shepherd-RL-7B \cite{Math-shepherd} && Mistral && 45.4 && 82.6 & 22.6 & 82.4 & 45.3 && 24.4 & 33.2 & 27.1 && ACL24 \\
            Xwin-Math-V1.1-7B \cite{Xwin-Math} && Llama 2 && 45.5 && 74.7 & 32.3 & 81.0 & 50.6 && 26.4 & 27.4 & 26.0 && Arxiv24\\
            WizardMath-V1.1-7B \cite{Wizardmath} && Mistral && 48.1 && 74.5 & 29.3 & 75.4 & 70.7 && 25.3 & 33.7 & 28.0 && Arxiv23\\
            Skywork-OR1-Math-7B \cite{skyworK} && Qwen2.5-Math && 55.1 && 58.5 & 80.3 & 49.6 & 84.3 && 50.0 & 35.0 & 55.0 && Arxiv25\\
            MathScale-7B \cite{MathScale} && Mistral && 61.0 && 75.6 & 37.8 & 80.6 & 75.7 && 44.5 & 62.3 & 50.6 && ICML24 \\
            DART-Math-7B (Prop2Diff) \cite{DART-Math} && DeepSeek-math && 63.1 && 87.4 & 52.6 & 83.8 & 73.7 && 38.6 & 56.8 & 48.5 && NeurIPS24 \\ 
            \midrule
            && && && \multicolumn{4}{c}{\textbf{PoT code generation}} & & \multicolumn{3}{c}{\textbf{CoT reasoning}} && \\
            \cmidrule(lr){7-10}\cmidrule(lr){11-14}
            % Baseline methods
            MLE && CodeLlama-7B && 45.2 && 59.4 & 34.0 & 72.0 & 40.7 && 37.8 & 34.1 & 38.6 && - \\
            MixCE \cite{MixCE} && CodeLlama-7B && 46.1 && 59.9 & 34.4 & \textbf{72.6} & 43.0 && 38.2 & 34.5 & \textbf{40.3} && ACL23 \\
            EMO \cite{EMO} && CodeLlama-7B && - && - & - & - & - && - & - & - && ICLR24 \\
            NTL-WAS \cite{NTL} && CodeLlama-7B && 47.4 && 59.7 & 35.0 & 72.3 & 44.4 && 39.0 & 42.7 & 38.5 && ICML25 \\
            DIST$^2$Loss \cite{DIST2loss} && CodeLlama-7B && 47.6 && 59.4 & 34.0 & 71.2 & \textbf{44.9} && 44.5 & 39.6 & 39.9 && ICLR26 \\
            \rowcolor{green!8}
            DEL (Ours) && CodeLlama-7B && \textbf{49.0} && \textbf{59.9} & \textbf{35.7} & 72.4 & 44.4 && \textbf{46.9} & \textbf{43.6} & 40.1 && - \\
            \addlinespace[0.5em]
            MLE && Qwen2.5-1.5B && 52.8 && 54.0 & 42.3 & 71.6 & 44.2 && 45.3 & 56.4 & 55.5 && - \\
            MixCE \cite{MixCE} && Qwen2.5-1.5B && 52.9 &&  52.8 & 42.7 & 72.7 & 43.2 && 47.2 & 54.5 & \textbf{57.3} && ACL23 \\
            EMO \cite{EMO} && Qwen2.5-1.5B && 53.4 && 54.5 & 42.6 & 73.5 & 47.7 && 48.0 & 52.7 & 55.1 && ICLR24 \\
            NTL-WAS \cite{NTL} && Qwen2.5-1.5B && 53.8 && 55.1 & \textbf{43.5} & 73.3 & 46.1 && 43.1 & 59.1 & 56.3 && ICML25 \\
            DIST$^2$Loss \cite{DIST2loss} && Qwen2.5-1.5B && 53.1 && 54.2 & 41.9 & 73.2 & 44.6 && 45.7 & 55.9 & 56.5 && ICLR26 \\
            \rowcolor{green!8}
            DEL (Ours) && Qwen2.5-1.5B && \textbf{55.4} && \textbf{57.0} & 42.6 & \textbf{74.0} & \textbf{48.1} && \textbf{48.8} & \textbf{60.9} & 56.1 && - \\
            \addlinespace[0.5em]
            MLE && Mistral-7B && 54.2 && 75.1 & 39.3 & 82.4 & 44.2 && 48.8 & 45.0 & 44.9 && - \\
            MixCE \cite{MixCE} && Mistral-7B && 55.1 && 75.6 & 39.4 & 81.9 & 45.9 && 49.4 & 47.1 & 46.6 && ACL23 \\
            EMO \cite{EMO} && Mistral-7B && - && - & - & - & - && - & - & - && ICLR24 \\
            NTL-WAS \cite{NTL} && Mistral-7B && 55.6 && 75.1 & 39.9 & \textbf{83.5} & 44.9 && 49.2 & 49.5 & \textbf{46.8} && ICML25 \\
            DIST$^2$Loss \cite{DIST2loss} && Mistral-7B && 54.6 && 74.8 & 40.0 & 81.4 & 43.6 && 50.8 & 47.7 & 43.7 && ICLR26 \\
            \rowcolor{green!8}
            DEL (Ours) && Mistral-7B && \textbf{56.5} && \textbf{76.0} & \textbf{40.3} & 81.7 & \textbf{46.9} && \textbf{53.1} & \textbf{51.8} & 45.5 && - \\
            \addlinespace[0.5em]
            MLE && DeepSeek-math-7B-Ins. && 64.4 && 76.8 & 51.4 & 80.9 & 52.7 && 59.1 & 69.1 & 60.9 && - \\
            MixCE \cite{MixCE} && DeepSeek-math-7B-Ins. && 64.6 && 77.2 & 51.2 & 81.2 & \textbf{54.9} && 55.9 & 70.1 & 61.9 && ACL23 \\
            EMO \cite{EMO} && DeepSeek-math-7B-Ins. && 64.8 && 77.6 & 50.9 & 80.7 & 54.7 && 57.9 & 69.5 & 62.1 && ICLR24 \\
            NTL-WAS \cite{NTL} && DeepSeek-math-7B-Ins. && 64.9 && 77.0 & 51.7 & 81.6 & 50.2 && 59.4 & 73.6 & 60.8 && ICML25 \\
            DIST$^2$Loss \cite{DIST2loss} && DeepSeek-math-7B-Ins. && 64.8 && 76.6 & 51.2 & 82.5 & 53.7 && 58.3 & 69.5 & 61.7 && ICLR26 \\
            \rowcolor{green!8}
            DEL (Ours) && DeepSeek-math-7B-Ins. && \textbf{66.1} && \textbf{77.6} & \textbf{51.9} & \textbf{83.0} & 52.5 && \textbf{61.0} & \textbf{74.5} & \textbf{62.1} && - \\
            \addlinespace[0.5em]
            MLE && Qwen2.5-7B && 68.3 && 81.4 & 51.6 & 87.0 & 51.0 && 63.8 & 76.8 & 66.3 && - \\
            MixCE \cite{MixCE} && Qwen2.5-7B && 68.3 && 81.6 & 55.3 & 86.9 & 53.5 && 61.4 & 72.7 & 66.5 && ACL23 \\
            EMO \cite{EMO} && Qwen2.5-7B && 68.5 && 81.4 & 55.8 & 87.6 & 50.2 && 63.0 & 75.0 & 66.2 && ICLR24 \\
            NTL-WAS \cite{NTL} && Qwen2.5-7B && 68.8 && 81.5 & 54.6 & 87.1 & 53.5 && 59.8 & \textbf{79.1} & 65.9 && ICML25 \\
            DIST$^2$Loss \cite{DIST2loss} && Qwen2.5-7B && 69.2 && 81.7 & 56.0 & 87.0 & 51.8 && 63.8 & 76.8 & 67.6 && ICLR26 \\
            \rowcolor{green!8}
            DEL (Ours) && Qwen2.5-7B && \textbf{70.6} && \textbf{81.7} & \textbf{56.0} & \textbf{87.6} & \textbf{57.0} && \textbf{65.4} & 78.6 & \textbf{67.9} && - \\
            \bottomrule
        \end{tabular}
    \end{threeparttable}
    \vspace{-14pt}
    \label{tab:seven-benchmarks}
\end{table*}

\mypara{Datasets}. We evaluate DEL on seven widely used mathematical reasoning benchmarks:
1) high quality grade school math problems dataset GSM8K \cite{GSM8K};
2) multi-level difficulty math competition dataset MATH \cite{MATH};
3) arithmetic math word problem dataset SVAMP \cite{SVAMP};
4) simultaneous equations dataset SimulEq \cite{SimulEq};
5) algebraic word problems with rationales AQuA \cite{AQuA};
6) general college entrance exams SAT-Math \cite{SAT-Math};
and 7) college math problems dataset MMLU \cite{MMLU}.
We report zero-shot in-context-learning task accuracy and the mean accuracy (mACC) over the evaluated benchmarks.
To evaluate the applicability, we adopt CoT reasoning for AQuA, SAT-Math, and MMLU, while we use program-of-thought (PoT) for code generation on GSM8K, MATH, SVAMP, and SimulEq.
If the code generation fails, we follow \cite{MAmmoTH} to use CoT reasoning for assistant.
Training is conducted on the large-scale high-quality math instruction tuning dataset MathInstruct \cite{MAmmoTH}, which consists of 13 math datasets covering a broad variety of math fields and complexity levels such as Algebra, Probability, NumTheory, Calculus, and Geometry.

\mypara{Implementation Details}. We adopt the Huggingface transformer library \cite{wolf2020transformers} to conduct the experiments.
We train all the models for 2 epochs.
The AdamW \cite{AdamW} is adopted as the optimizer with bfloat16 precision enabled.
We use a learning rate of 2e-5 for CodeLlama, and 5e-6 for Mistral, Qwen2.5, and DeepSeek-Math-Instruct.
The learning rate decay adopts a cosine scheduler with a warm-up strategy for the first 3\% training steps.
The hyperparameter $\alpha$ in \Eqref{eq:del-w} is tuned separately for LLMs due to their different vocabulary size, which is set to $0.1$ for CodeLlama and Mistral, and $0.01$ for Qwen2.5 and DeepSeek-Math-Instruct.
We set $k=2$ and $\beta=1.02$ in the floating-point number optimization by default.
The ablation for these parameters can be found in \appnref{appendix:ablation}.
All the trainings are conducted on 8 NVIDIA A100 GPUs with batchsize 1 per GPU and gradient accumulation step 32, thus resulting in a total batch size of 256.

\subsection{Experimental Results}\label{sec:results}

This subsection shows the results of DEL on four PoT code generation settings (GSM8K, MATH, SVAMP, and SimulEq) and three CoT reasoning settings (AQuA, SAT, and MMLU).
For other evaluations, we use Qwen2.5-1.5B by default.
In \appnref{appendix:visualization}, we provide more exempler visual comparisons for various numerical learning methods.

\begin{table*}[t]
\scriptsize
    \caption{Numerical answer evaluation on GSM8K with various loss functions. We report the numerical distance and the task accuracy for the predicted number answers, along with the per digit metric for the unit place, tens place, hundreds place, thousands place, and ten thousands place, denoted by $1$, $10^1$, $10^2$, $10^3$, $10^4$, respectively.}
    \vspace{-4pt}
  \subfloat[Numerical distance evaluation.\label{tab:numerical-distance}]{
  \tablestyle{3.5pt}{1.05}
        \begin{tabular}{lcccccccc}
            \toprule
            \multirow{2}{*}{\textbf{Method}} & \multicolumn{5}{c}{\textbf{MAE per digit}} & \multirow{2}{*}{\textbf{Global MAE}} \\
            \cmidrule(lr){2-6}
            & 1 & $10^1$ & $10^2$ & $10^3$ & $10^4$ & \\
            \midrule
            MLE & 1.26 & 1.20 & 0.91 & 1.11 & 1.57 & 0.72 \\
            MixCE \cite{MixCE} & 1.25 & 1.18 & 0.91 & 1.08 & 1.72 & 0.79 \\
            EMO \cite{EMO} & 1.13 & 1.08 & 0.85 & 0.98 & 1.72 & 0.71 \\
            NTL \cite{NTL} & \textbf{1.16} & 1.06 & 0.84 & 1.05 & 1.53 & 0.78 \\
            DIST$^2$Loss \cite{DIST2loss} & 1.18 & 1.14 & 0.93 & 1.11 & 1.62 & 0.77 \\
            DEL (Ours) & 1.17 & \textbf{1.03} & \textbf{0.81} & \textbf{0.91} & \textbf{1.17} & \textbf{0.63} \\
            \bottomrule
        \end{tabular}}\hfill
  \subfloat[Numerical accuracy evaluation.\label{tab:numerical-accuracy}]{
  \tablestyle{3.5pt}{1.05}
        \begin{tabular}{lcccccccc}
            \toprule
            \multirow{2}{*}{\textbf{Method}} & \multicolumn{5}{c}{\textbf{Accuracy per digit}} & \multirow{2}{*}{\textbf{Task Acc.}} \\
            \cmidrule(lr){2-6}
            & 1 & $10^1$ & $10^2$ & $10^3$ & $10^4$ & \\
            \midrule
            MLE & 66.8 & 62.4 & 63.7 & 66.4 & 61.7 & 54.0 \\
            MixCE \cite{MixCE} & 66.6 & 62.5 & 65.6 & 65.6 & 57.4 & 52.8 \\
            EMO \cite{EMO} & 69.3 & 64.1 & 66.3 & 66.4 & 59.6 & 54.5 \\
            NTL \cite{NTL} & 68.9 & 65.1 & 64.6 & 64.9 & 61.7 & 55.1 \\
            DIST$^2$Loss \cite{DIST2loss} & 67.9 & 63.5 & 64.6 & 61.1 & 53.2 & 54.2 \\
            DEL (Ours) & \textbf{70.0} & \textbf{67.0} & \textbf{67.2} & 64.9 & \textbf{61.7} & \textbf{57.0} \\
            \bottomrule
        \end{tabular}}
    \vspace{-10pt}
    \label{tab_3}
\end{table*}

\begin{table*}[t]
    \renewcommand\arraystretch{1}
    \centering
    \footnotesize
    \setlength{\tabcolsep}{4pt}
    \caption{Evaluation of diverged components in numerical learning.}
    \vspace{-4pt}
    \begin{threeparttable}
        \begin{tabular}{ccccccccc}
            \toprule
            Digit probability $\hat{\mathbf{p}}(\cdot|x_{<t})$ & DEL in \Eqref{eq:del} & Distance term in NTL & Floating-point number optim. & mACC \\
            \midrule
            & & & & 52.9 \\
            \checkmark & & \checkmark & & 53.8 \\
            \checkmark & & \checkmark & \checkmark & 54.4 \\
            & \checkmark & & & 55.1 \\
            & \checkmark & \checkmark & & 54.3 \\
            & \checkmark & & \checkmark & 55.4 \\

            \bottomrule
        \end{tabular}
    \end{threeparttable}
    \vspace{-10pt}
    \label{tab:diverged-components}
\end{table*}

\mypara{Benchmark Evaluation}.
We report the quantitative results in \tabref{tab:seven-benchmarks}.
One can see that DEL produces the best results in most benchmarks and achieves the best average performance across all model architectures, demonstrating consistent improvements over existing methods.
It can be seen that the general methods MixCE and EMO can improve mathematical reasoning performance by enhancing logical deduction and language understanding capability.
While EMO achieves better overall results than MixCE, it fails in training on CodeLlama and Mistral.
This is presumably because that EMO relies on the token embeddings to calculate the similarity between tokens, which may introduce numerical instability in training.
The numerical learning methods NTL and DIST$^2$Loss better boost the accuracy since they are designed to improve digit prediction performance.
Our DEL better caters to the low-entropy nature of digits, resulting in the best numerical learning performance.

\mypara{Numerical answers evaluation}.
Aside from global task accuracy, we also compare the numerical answers for the methods.
We consider two evaluation metric.
In \tabref{tab:numerical-distance}, we calculate the global numerical distance for the predicted numbers, as well as for every digital place.
To alleviate the effect of large numbers, we follow \cite{NTL} to use the log Mean Absolute Error (MAE) to measure the relative distance between the predicted digits and the ground-truths.
We omit reporting very large digital place because the model may predict excessively long digital sequences, and these numbers are rare.
In \tabref{tab:numerical-accuracy}, we report the numerical accuracy for the predicted numbers, as well as for every digital place.
The results shows that our DEL can substantially reduce the numerical distance even without the explicit injection of numerical distance term, demonstrating a more precise number generation.

\mypara{Evaluation of diverse components in numerical learning}.
We conduct experiments to check out the effect of each component in numerical learning.
As shown in \tabref{tab:diverged-components}, we start from MLE, a clean baseline of 52.9 mACC on seven benchmarks.
The second row is NTL \cite{NTL}, which adopts the digital probability $\hat{\mathbf{p}}(\cdot|x_{<t})$ as the criterion term, along with the numerical distance, achieving a mACC of 53.8.
It can be further improved to 54.4 if we intergrate our floating-point number generation into NTL, as shown by the third row.
Then, in 4-th row, our DEL without any other components produces 55.1 mACC, demonstrating the superiority of the new criterion term.
If we add the distance term of NTL, it degrades to 54.3, as shown in the results of 5-th row.
As we have analyzed in \secref{sec:DEL}, the numerical distance terms injects a strong geometric prior into the LLMs, which has a deviation from the semantic information of digits.
Thus we deprecate the distance term in our formulation.
Finally, combined with the floating-point optimization, the performance can be further improved to 55.4 mACC (see the last row).
This gives us our final DEL formula (\Eqref{eq:del-w})

\mypara{Number Uncertainty}.
To investigate the uncertainty of the predicted numbers, we further statistically analyze the token entropy of the predicted sequence.
Based on the full space probability $\mathbf{p}(\cdot|x_{<t})$, we decompose the entropy calculation into two disjoint parts, the digital part and the non-digital part.
As shown in \figref{fig:digital-entropy}, the orginal baseline MLE produces a low digital entropy, approximately 0.1.
NTL \cite{NTL} significantly reduces digital entropy because its optimization goal is to forcibly push the probability of non-target digits to $0$.
Another extreme example is DIST$^2$Loss \cite{DIST2loss}, which significantly increases the digital entropy, making it even exceed $1$.
This is because DIST$^2$Loss takes the distance-driven soft targets as optimization objectives.
It should be noted that the average entropy of non-digit tokens is around $0.3\sim 0.5$ for all these methods, as depicted by the orange line in \figref{fig:digital-entropy}.
This indicates that DIST$^2$Loss largely rises the uncertainty of number generation, whereas NTL tends to significantly sharpen the probability distribution of digits.
The criterion term of DEL makes it a mild method for optimizing number uncertainty.

\begin{figure*}[!tb]
	\centering
    \small
		\begin{overpic}[width=1\linewidth]{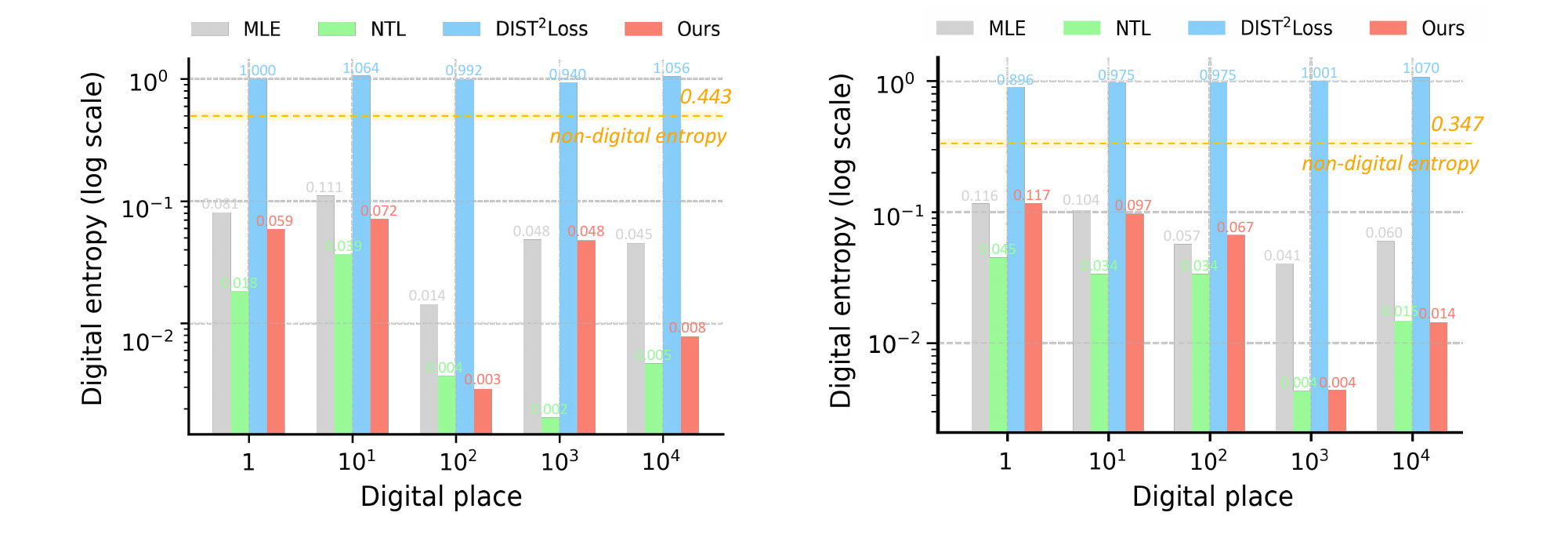}
        \put(15,1){(a) Digital entropy on SVAMP}
        \put(62,1){(b) Digital entropy on GSM8K}
  \end{overpic}
  \vspace{-16pt}
	\caption{Comparison of digital entropy between various numerical learning methods.}
  \label{fig:digital-entropy}
  \vspace{-12pt}
\end{figure*}

\section{Conclusion}

In this paper, we analyzed the existing numerical learning approaches in large language model training, and found that state-of-the-art methods generally follow the criterion-distance formulation.
From the perspective of entropy optimization, we presented digit entropy loss (DEL) to achieve better performance on number generation.
We further generalized the numerical learning to floating-point number optimization, enabling the language models with holistic numerical awareness.
Experiments on seven mathematical reasoning benchmarks demonstrated the effectiveness of the proposed DEL over prior state-of-the-art numerical learning methods.

\mypara{Limitations}.
While demonstrating consistent improvements over previous methods, DEL has some limitations.
Firstly, it is still optimized in a token-level manner, without joint modeling of the global numerical semantics of the number sequence.
Secondly, DEL currently extends numerical learning to floating-point number optimization, but has not yet expanded many mathematical operations.
In the future, one may consider modeling mathematical operations in LLMs, such as addition, multiplication, division, exponentiation, etc., enabling LLMs with more numerical reasoning capabilities.

%==========================================================================
% Acknowledgements
%==========================================================================
\section*{\textcolor{polyured}{$\blacksquare$}\,Acknowledgements}
This work was supported by the Visual Computing Lab at The Hong Kong
Polytechnic University and VCIP at College of Computer Science, Nankai University.
We thank our colleagues for valuable discussions.

%==========================================================================
% Bibliography
%==========================================================================
{\small
\bibliography{ref}
}

\appendix

\clearpage
\setcounter{table}{0}
\setcounter{equation}{0}
\setcounter{figure}{0}
\renewcommand{\thetable}{\thesection.\arabic{table}}
\renewcommand{\theequation}{\thesection.\arabic{equation}}
\renewcommand{\thefigure}{\thesection.\arabic{figure}}
\begin{center}
    {\LARGE\bfseries Appendix}
\end{center}

In appendix, we provide the following materials:
\begin{enumerate}[label=\Alph*.]
    \item More analysis for various numerical learning methods (referring to \secref{sec:3} in the main paper).

    \item Hyper-parameters study for the loss weight $\alpha$ and the designs in floating-point number optimization (referring to \secref{sec:setup} in the main paper).

    \item Exempler visual comparisons for the numerical learning methods, along with our failure case analysis (referring to \secref{sec:results} in the main paper).
\end{enumerate}

\section{More Analysis}\label{appendix:analysis}

\subsection{Notations}

We here give more analysis for the numerical learning.
We first lay out some important notations.
For simplicity, we consider one multi-digit floating-point number, which is given by \Eqref{eq:float-point} and we write here again:
\begin{equation}\label{eq:number}
    \mathbf{g}=\mathbf{g}_0^L\mathbf{g}_1^{L-1}\cdots \mathbf{g}_L^0\textbf{.}\mathbf{g}_{L+2}^{-1}\mathbf{g}_{L+3}^{-2}\cdots \mathbf{g}_{L+M+1}^{-M},
\end{equation}
where the superscripts and the subscripts of $\mathbf{g}$ represent the digital place and the generation step, respectively.
The total sequence $U$ with length $T$ can be represented as
\begin{equation}
    U = U_{pre}, \mathbf{x}= U_{pre},x_0,x_1,\cdots,x_{L},\cdots, x_{L+M+1},
\end{equation}
where $U_{pre}$ denotes the input sequence preceding the number $\mathbf{x}$.
We denote $x_{<t}$ as all preceding input sequences of $x_t$.
The numerical sequence $\mathbf{g}$ can be decoded from the digital token sequences,
\begin{equation}
    \mathbf{g}=\text{Decode}(\mathbf{x}).
\end{equation}

\tabref{tab:notations} presents the notations we used in the analysis.

\subsection{Analysis to various numerical learning methods}

We here give the analysis to various numerical learning methods.
We first simplify the symbols by ignoring the token index $t$ since all losses NTL \cite{NTL}, DIST$^2$Loss and DEL are token-wised and independent of $t$.
We lay out the penalty terms for all the three losses.
\begin{equation}
\begin{aligned}
    &\mathcal{R}_{NTL} = \hat{\mathbf{p}}||V_d-\mathbf{g}||_1 = \sum\limits_k\hat{\mathbf{p}}_kd_k, \quad d_k = ||k-\mathbf{g}||_1, \\
    &\mathcal{R}_{dist} = \mathcal{L}_{\text{KL}}(\mathbf{q}\parallel\hat{\mathbf{p}}) = -\sum\limits_k\mathbf{q}_k\log\hat{\mathbf{p}}_k - \mathcal{H}(\mathbf{q}),\\
    &\mathcal{R}_{DEL} = \hat{\mathbf{p}}\mathbf{BCE}(\dot{\mathbf{p}}, y) = \sum\limits_k\hat{\mathbf{p}}_k B_k= \sum\limits_k\hat{\mathbf{p}}_k \underbrace{\Big(-y_k\log\dot{\mathbf{p}}_k - (1-y_k)\log\big(1-\dot{\mathbf{p}}_k\big)\Big)}_{B_k}.
\end{aligned}
\end{equation}
The partial derivative of the penalty term w.r.t the output logits $\hat{z}_k$ are as follows:
\begin{equation}
\begin{aligned}
    & \frac{\partial\mathcal{R}_{NTL}}{\partial z_k} = \hat{\mathbf{p}}_k(d_k-\mathcal{R}_{NTL}), \\
    & \frac{\partial\mathcal{R}_{dist}}{\partial z_k} = \hat{\mathbf{p}}_k-\mathbf{q}_k, \\
    & \frac{\partial\mathcal{R}_{DEL}}{\partial z_k} = \hat{\mathbf{p}}_k(B_k-\mathcal{R}_{DEL}+\dot{\mathbf{p}}_k-y_k)
\end{aligned}
\end{equation}

\begin{proposition}\label{pro1}
The optimal solution of $\mathcal{R}_{DEL}$ is the one-hot label $y$.
\end{proposition}
\begin{proof}
The one-hot label $y$ satisfies $y_\mathbf{g}=1$ and $y_k=0,\,\forall k\neq \mathbf{g}$.
Since $\mathbf{BCE}(\dot{\mathbf{p}}, y)$ has the global minima $0$ when $\dot{\mathbf{p}}=y$, and by nonnegativity of $\hat{\mathbf{p}}$, one can get $\mathcal{R}_{DEL}$ has the global minima $0$, which is also one of the global minima of the Shannon entropy $\mathcal{H}(\hat{\mathbf{p}})$.
\end{proof}

\mypara{Analysis to NTL.} We can derive from its partial dericative that NTL encourages to lift up the probabilities for the tokens near the target $\mathbf{g}$.
Let $s_k = d_k - \mathcal{R}_{NTL}$, we have the following inequalities,
\begin{equation}
\begin{aligned}
    s_k = d_k - \sum\limits_{i\neq \mathbf{g}} \hat{\mathbf{p}}_i d_i \leq d_k- \sum\limits_{i\neq \mathbf{g}} \hat{\mathbf{p}}_i = d_k -1 + \hat{\mathbf{p}}_\mathbf{g},\\
    s_k = d_k - \sum\limits_{i\neq \mathbf{g}} \hat{\mathbf{p}}_i d_i \geq d_k- \sum\limits_{i\neq \mathbf{g}} 9\hat{\mathbf{p}}_i = d_k -9 + 9\hat{\mathbf{p}}_\mathbf{g}.\\
\end{aligned}
\end{equation}
Consequently, we have
\begin{equation}
    d_k -9 + 9\hat{\mathbf{p}}_g\leq s_k\leq d_k -1 + \hat{\mathbf{p}}_g
\end{equation}
This indicates that the logits whose tokens is furthest from $\mathbf{g}$ with distance $d_k=9$ will be definitely reduced since $s_k\geq 9-9+9\hat{\mathbf{p}}_g>0$ and thus the partial derivative $\partial\mathcal{R}_{NTL}/\partial z_k>0$.
Likewise, the logits $z_\mathbf{g}$ will be increased since $s_k\leq 0-1+\hat{\mathbf{p}}_\mathbf{g} = \hat{\mathbf{p}}_\mathbf{g} - 1 < 0$ and thus the partial derivative $\partial\mathcal{R}_{NTL}/\partial z_k<0$.
For the logits with medium distance, e.g., $d_k\in[1,8]$, will be optimized according to $\hat{\mathbf{p}}_g$.
It encourages to increase the logits of the token near the target $\mathbf{g}$.
When $\hat{\mathbf{p}}_\mathbf{g} >8/9$, indicating the probability of the target token is optimized well, it deduces $s_k>0$ and all the logits of the non-target tokens will be further reduced.
To sum up, NTL gradually shifts the probability mass from tokens far from $\mathbf{g}$ to the tokens near $\mathbf{g}$, and finally $\mathbf{g}$ alone.

\mypara{Analysis to DIST$^2$Loss}.
By the property of KL divergence, the optimal solution of $\hat{\mathbf{p}}$ equals to the soft target $\mathbf{q}$.
In this case, one can obtain the digital conditional entropy, 
\begin{equation}
    \mathcal{H}(\hat{\mathbf{p}})=\mathcal{H}(\mathbf{q}) = -\sum\limits_k \mathbf{q}_k\log\mathbf{q}_k
\end{equation}
Since $\mathcal{H}(\mathbf{q})$ is a constant, it can be directly computed.
There is totally 10 possible ground-truth digits for $\mathbf{g}$.
Due to the symmetry of $\mathbf{q}$, it can be reduced to 5 cases.
We list the computation results of $\mathcal{H}(\mathbf{q})$ in \tabref{tab:Hq}.
One can see that $0.6456\leq\mathcal{H}(\mathbf{q})\leq 1.0715$, which is much larger than the token entropy as we have shown in \figref{fig:digital-entropy}.
While the cross-entropy term $\mathcal{H}(\mathbf{p},\mathbf{g})$ attempts to optimize $\mathbf{p}$ to an one-hot distribution, the penalty term of DIST$^2$Loss still enlarges the digital entropy because the target itself has a large entropy.
This is why we observe the phenomenon of increasing digital entropy from \figref{fig:digital-entropy} that the digital entropy of DIST$^2$Loss is approximately 1.0.

\begin{table*}[t]
    \renewcommand\arraystretch{1}
    \centering
    \footnotesize
    \setlength{\tabcolsep}{4pt}
    \caption{Notations.}
    \begin{threeparttable}
        \begin{tabular}{ll}
            \toprule
            Notations &  Interpretation \\
            \midrule
            $\hat{z}\in\mathbb{R}^{L+M+2,10}$ & All the predicted logits for $\mathbf{x}$ on digital tokens. \\
            $\hat{z}_t\in\mathbb{R}^{10}$ & The $t$-th logit vector of $x_t$, $t\in \{0,1,\cdots,L+M+1\}$. \\
            $\mathbf{g}_t$ & The numerical value of the token $x_t$. \\
            $y_t$ & The one-hot label of $\mathbf{g}_t$. \\
            $V_d=\{0,1,\cdots,9\}$ & The digital vocabulary. \\
            $\hat{\mathbf{p}}(\cdot|x_{<t})=\mathcal{S}(\hat{z}_t)$ & The digital conditional probability of $x_t$ over $V_d$, calculated by the softmax function $\mathcal{S}$. \\
            $\dot{\mathbf{p}}(\cdot|x_{<t})=\sigma(\hat{z}_t)$ & The binary digital conditional probability of $x_t$, calculated by the sigmoid function $\sigma$. \\
            $\mathbf{q}=\mathcal{S}(-\rho(V_d, \mathbf{g}_t), \tau)$ & The soft distribution of numerical distance $\rho(V_d, \mathbf{g}_t)$. \\
            $\mathcal{S}(\cdot, \tau)$ & The Softmax function with temperature $\tau$. By default, $\tau=1$.  \\
            $\sigma(\cdot)$ & The Sigmoid function. \\
            \bottomrule
        \end{tabular}
    \end{threeparttable}
    \vspace{0pt}
    \label{tab:notations}
\end{table*}

\begin{table*}[t]
    \renewcommand\arraystretch{1}
    \centering
    \small
    \setlength{\tabcolsep}{6pt}
    \caption{The entropy of the distance distribution $\mathbf{q}$.}
    \begin{threeparttable}
        \begin{tabular}{lccccc}
            \toprule
            The target token $\mathbf{g}$ & 0, 9 & 1, 8 & 2, 7 & 3, 6 & 4, 5 \\
            $\mathcal{H}(\mathbf{q})$ & 0.6456 & 1.0169 & 1.0708 & 1.0708 & 1.0715 \\
            \bottomrule
        \end{tabular}
    \end{threeparttable}
    \vspace{-10pt}
    \label{tab:Hq}
\end{table*}

\mypara{Analysis to Digit Entropy Optimization}.
Digit entropy optimization is not new.
For example, it has been independently discussed in \cite{bian2025feature} for object detection and in \cite{agarwal2025the} for LLM reasoning, where entropy minimization is observed to be useful.
Others take the opposite approach: maximum entropy principle \cite{li2025preserving,wangbeyond} to encourage the models to evolve more diversity of thought chains.
To chase for deterministic number generation, we consider entropy minimization, as presented in \Eqref{eq:entropy}.
In the 5-th row of \tabref{tab:focal-vs-DEL}, one can see that the overall performance of digit entropy optimization already surpasses previous numerical learning methods, e.g., NTL and DIST$^2$Loss.
Due to its unsupervised nature, it is free from the issue of numerical distance.
However, this entropy loss has many trivial solutions.
Any one-hot distribution can reach the global minima.
Our DEL leverages the digit conditional probability and the binary cross-entropy, enabling a supervised entropy optimization manner.
The last row of \tabref{tab:focal-vs-DEL} show that our DEL achieves the best mACC on the 7 benchmarks, demonstrating the effecttiveness of DEL.

\mypara{Analysis to Focal Loss}.
Adopting binary cross-entropy loss in multi-class learning is not new either.
It has been recently demonstrated in \cite{li2025bce} that BCE can maximize the intra-class compactness and inter-class distinctiveness among sample features, which makes it has potential to achieve better classification performance than standard cross-entropy loss.
Early research focal loss \cite{lin2017focal} also adopted BCE as the core objectives in object detection.
It was originally proposed to alleviate the gradient dominated problem caused by a large number of negative samples.
This loss function was designed to balance the learning of positive and negative anchor boxes.
Under our defined symbol system, integrating focal loss into numerical learning can be written as
\begin{equation}\label{eq:focalloss}
    \mathcal{L}_{FL} = \alpha \sum\limits_k A_k^\gamma B_k = \alpha\sum\limits_k\Big(y_k(1-\dot{\mathbf{p}}) + (1-y_k)\dot{\mathbf{p}}\Big)^\gamma \Big(-y_k\log\dot{\mathbf{p}}_k - (1-y_k)\log\big(1-\dot{\mathbf{p}}_k\big)\Big)
\end{equation}
It can be seen that focal loss models each token as a binary classification task on $|V|$ classes.
Without the softmax-based probabilities, it is difficult to modeling the inner relationship of digits and is inconsistent with the inference phase.
Moreover, the design principles of focal loss is opposite to our DEL.
Focal loss significantly reduces the loss for easily classified samples by setting the weight $A_k$, i.e., $A_k\rightarrow 0$ when $y_k=1$ and $\dot{\mathbf{p}}_k\rightarrow 1$.
In contrast, DEL focuses more on the high-confidence tokens.
To support focal loss in LLM training, we make the necessary modifications to enable it to perform in digit subspace.
In \tabref{tab:focal-vs-DEL}, we compare our DEL with the digital focal loss.
The experimental results demonstrate that DEL can achieve better accuracy on 6 of 7 benchmarks.

\begin{table*}[!tb]
    \renewcommand\arraystretch{1.1}
    \centering
    \small
    \setlength{\tabcolsep}{3.5pt}
    \caption{Quantitative comparison between digital focal loss and our DEL. The base model is Qwen2.5-1.5B.}
    \begin{threeparttable}
        \begin{tabular}{lcccccccccc}
            \toprule
            \textbf{Method} & \textbf{mACC} && \textbf{GSM8K} & \textbf{MATH} & \textbf{SVAMP} & \textbf{SimulEq} & \textbf{AQuA} & \textbf{SAT} & \textbf{MMLU} \\
            \midrule
            MLE & 52.8 && 54.0 & 42.3 & 71.6 & 44.2 & 45.3 & 56.4 & 55.5 \\
            Digital Focal Loss & 53.0 && 55.9 & 43.0 & 72.8 & 44.7 & 45.3 & 54.5	& 54.8 \\
            NTL-WAS \cite{NTL} & 53.8 && 55.1 & \textbf{43.5} & 73.3 & 46.1 & 43.1 & 59.1 & 56.3 \\
            DIST$^2$Loss \cite{DIST2loss} & 53.1 && 54.2 & 41.9 & 73.2 & 44.6 & 45.7 & 55.9 & 56.5  \\
            Digit Entropy Optim. (\Eqref{eq:entropy}) & 54.2 && 54.6 & 42.4 & 72.7 & 47.7 & 47.2 & 60.5 & 54.2 \\
            DEL (Ours) & \textbf{55.4} && \textbf{57.0} & 42.6 & \textbf{74.0} & \textbf{48.1} & \textbf{48.8} & \textbf{60.9} & \textbf{56.1} \\
            \bottomrule
        \end{tabular}
    \end{threeparttable}
    \vspace{-4pt}
    \label{tab:focal-vs-DEL}
\end{table*}

\section{Hyper-parameters study}\label{appendix:ablation}

\mypara{Loss weight $\alpha$ in DEL}.
We investigate the impact of loss function weights on model performance.
\tabref{tab:alpha} reports the quantitative results.
Due to the varying vocabulary sizes of different LLMs, such as 151,936 for Qwen 2.5, 102,400 for DeepSeek-Math-Instruct, 32,016 for CodeLlama, and 32,000 for Mistral, the setting of $\alpha$ will be adjusted accordingly.
We set $\alpha=0.01$ for Qwen2.5 and DeepSeek-Math-Instruct and $\alpha=0.1$ for CodeLlama and Mistral in other experiments by default.

\mypara{$k$ and $\beta$ in floating-point number optimization}.
In \Eqref{eq:digital-place-weight}, we assign a linear penalty for integers and an exponential decay for decimals.
According to the results in \tabref{tab:k}, we set $k=2$ by default in all the other experiments.
As for the decimals, we conduct various options.
\tabref{tab:beta} shows us interesting results that the performance slightly drops if we set no penalty or a constant penalty for the decimals.
What if we treat the decimal as a new integer (as done by DIST$^2$Loss, see \figref{fig:comparison-digital-place}), the performance will degrade significantly.
This indicates that the discontinuous penalty between integers and decimals leads to suboptimal results.
In contrast, our floating-point number optimization can better match the continuity of numerical values, further boosting the performance gains.
We set $\beta=1.02$ by default.

\section{Visualization}\label{appendix:visualization}

We provide several visual comparison for the numerical learning methods.
We highlight the reasoning errors in \textcolor{red}{\textbf{red}} bold.
We use Qwen2.5-1.5B to inference.
In \figref{fig:visualization1}, we show that accurate number prediction is crucial to mathematical reasoning as a small error in number prediction would lead to a complete wrong answer to the mathematical problem.
More examples can be seen in \figref{fig:visualization2}, \figref{fig:visualization3}, \figref{fig:visualization4}.

We also present some failure cases that the current numerical learning is hard to handle.
The first one is irrelevant context distraction \cite{shi2023large}, as shown in \figref{fig:failure-cases}.
This error is irrelevant to number prediction accuracy, which is simply caused by the distracted attention to the irrelevant contexts.
The second error is calculation error.
The reasoning is correct, however, the calculation results are wrong.
As we have concluded in \secref{sec:3}, numerical reasoning requires language models to possess a range of capabilities, such as logical reasoning, symbolic comprehension, unit conversion, problem decomposition, contextual semantic association.
We believe that in the future, explicitly injecting more capabilities into the optimization process of language models will help improve numerical reasoning.

\begin{table*}[!tb]
\footnotesize
    \caption{\textbf{Ablations}. We show ablation experiments for DEL.}
  \subfloat[Loss weight $\alpha$ in DEL.\label{tab:alpha}]{
  \tablestyle{3pt}{1.0}
        \begin{tabular}{lcccccccccccc}
            \toprule
            $\alpha$ & & \textbf{Accuracy} & \textbf{MAE} \\
            \midrule
            %0.05 && 48.2 & 40.5 \\
            0.02 && 56.0 & 0.69  \\
            0.01 &&  \textbf{57.2} & \textbf{0.65}  \\
            0.005 &&  56.8 & 0.66 \\
            0 && 54.0 & 0.71  \\
            \bottomrule
        \end{tabular}}\hfill
  \subfloat[Linear weighting $k$ in integers in floating-point number optimization.\label{tab:k}]{
  \tablestyle{6.5pt}{1.0}
        \begin{tabular}{lcccccccccccc}
            \toprule
            $k$ & & \textbf{Accuracy} & \textbf{MAE} \\
            \midrule
            0.5 && 56.5 & 0.68  \\
            1 && 56.8 & 0.68  \\
            2 && \textbf{57.2} & \textbf{0.65} \\
            5 &&  56.2 & 0.72  \\
            \bottomrule
        \end{tabular}}\hfill
  \subfloat[Exponential decay $\beta$ in decimals in floating-point number optimization.\label{tab:beta}]{
  \tablestyle{4.5pt}{1.0}
        \begin{tabular}{lcccccccccccc}
            \toprule
            $\beta$ & & \textbf{Accuracy} & \textbf{MAE} \\
            \midrule
            $\infty$ && 56.4 & 0.69  \\
            1 && 56.6 & 0.66  \\
            1.02 && \textbf{57.2} & \textbf{0.65}  \\
            as integers &&  55.2 & 0.74  \\
            \bottomrule
        \end{tabular}}\hfill
    \vspace{0pt}
\end{table*}

\begin{figure*}[!tb]
	\centering
    \small
		\begin{overpic}[width=1\linewidth]{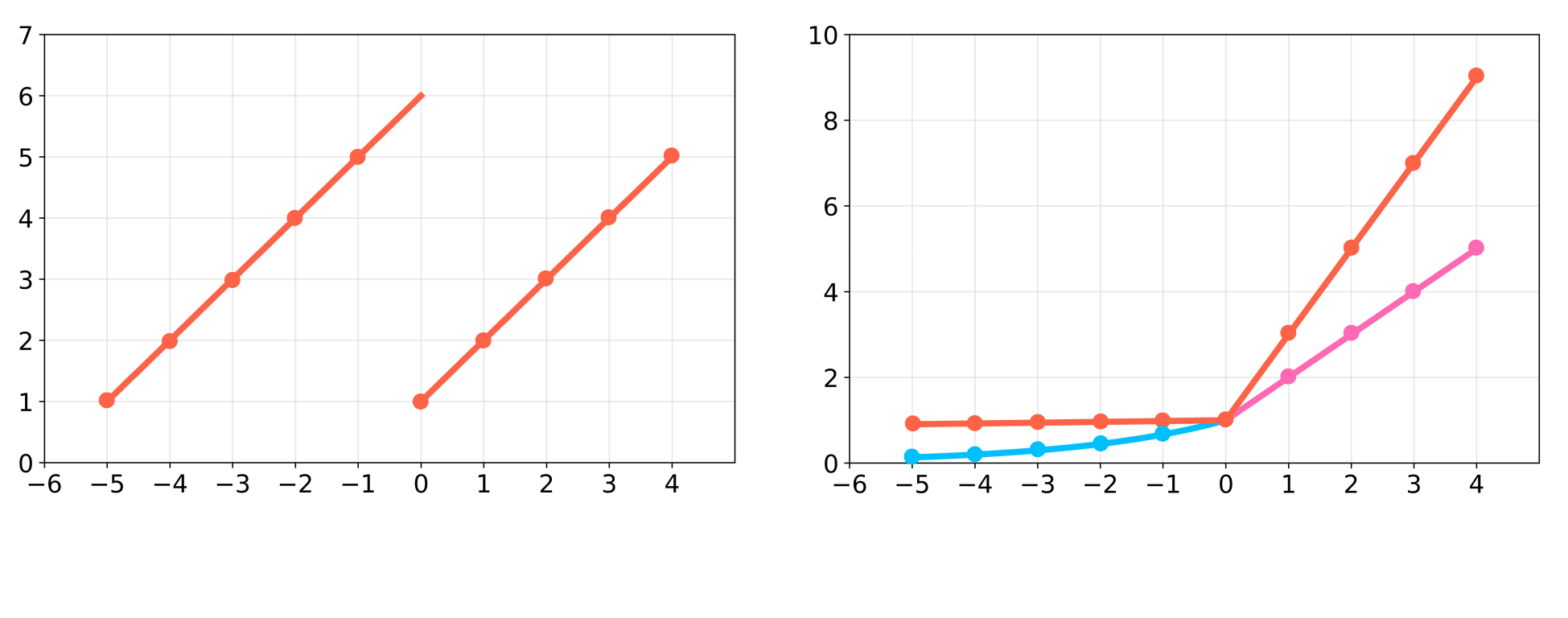}
        \put(24.5,5){$i$}
        \put(76,5){$i$}
        \put(13.5,2){(a) $u(t)$ in DIST$^2$Loss}
        \put(66,2){(b) $u(t)$ in our method}
  \end{overpic}
  \vspace{-16pt}
	\caption{Illustration of floating-point optimization. Our DEL adopts linear penalty for integers ($i\geq0$) and exponential decay in decimals ($i<0$), while previous work DIST$^2$Loss has a discontinuous penalty. Different colors represent different parameter choices in $u(t)$.}
  \label{fig:comparison-digital-place}
  \vspace{-4pt}
\end{figure*}

\begin{figure*}
  \linespread{1.05}\selectfont
  \begin{tcolorbox}
  [colback=white!95!gray, colframe=gray!50!black, rounded corners, title={Question: $$3x+x+x+x-3-2=7+x+x$$In the equation above, what is the value of $x$ ?\\
  Answer Choices: (A) $-\frac{5}{7}$ \quad (B) 1 \quad (C) $\frac{12}{7}$  \quad (D) 3\\
  Correct: D}]
  \textbf{MLE}:

    We can combine like terms on the left side of the equation to get
    $$3x+x+x+x-3-2=7+x+x$$
    $$\textcolor{red}{\textbf{4}}x-5=7+2x$$
    We can subtract $2x$ from both sides to get
    $$2x-5=7$$
    Then we can add 5 to both sides to get
    $$2x=12$$
    Finally, we can divide both sides by 2 to get
    $$x=6$$
    The answer is B\\

  \textbf{NTL}:

    3x+x+x+x-3-2=7+x+x
    
    3x+x+x+x-5=7+2x
    
    3x+x+x+x-2x=7+5
    
    \textcolor{red}{\textbf{3}}x=12
    
    x=12/3
    
    x=4
    
    The answer is B\\

  \textbf{DIST$^2$Loss}:
    $$3 x+x+x+x-3-2=7+x+x$$
    $$\textcolor{red}{\textbf{4}}x-5=7+2x$$
    $$2x=12$$
    $$x=6$$
    The answer is D\\
    
  \textbf{DEL (Ours)}:
  
    3x+x+x+x-3-2=7+x+x
    
    3x+x+x+x-5=7+2x
    
    3x+x+x+x-2x=7+5
    
    3x+x=12
    
    4x=12
    
    x=3
    
    The answer is D
\end{tcolorbox}
\vspace{-10pt}
  \caption{Visualization of prediction results of LLM. The reasoning errors are highlighted in \textcolor{red}{\textbf{red}} bold.}
  \label{fig:visualization1}
\end{figure*}

\begin{figure*}
  \linespread{1.2}\selectfont
  \begin{tcolorbox}
  [colback=white!95!gray, colframe=gray!50!black, rounded corners, title={Question: John plans to sell all his toys and use the money to buy video games. He has 13 lego sets and he sells them for \$15 each. He ends up buying 8 video games for \$20 each and has \$5 left. How many lego sets does he still have?\\
  Correct: 2}]
  \textbf{MLE}:
  
He earned 13*15=195 dollars from selling his lego sets

He spent 8*20=160 dollars on video games

He has 195-160-5=30 dollars left

\textcolor{red}{\textbf{He originally had 30+160=190 dollars}}

He has 190/15=12.666667 lego sets left

The answer is 12\\

  \textbf{NTL}:

He earned 13*15=195 dollars

He spent 8*20=160 dollars on video games

He has 195-160-5=30 dollars left

He has 195-30=165 dollars left to buy lego sets

\textcolor{red}{\textbf{He has 165/15=11 lego sets left}}

The answer is 11\\

  \textbf{DIST$^2$Loss}:

He earned 13*15=195 dollars from selling his lego sets

He spent 8*20=160 dollars on video games

He has 195-160-5=30 dollars left

He had 195-30=165 \textcolor{red}{\textbf{lego sets}}

He sold 195-160=35 lego sets

He has 165-35=130 lego sets left

The answer is 130\\

  \textbf{DEL (Ours)}:
  
He sold 13 lego sets for 13*15=195 dollars

He spent 8*20=160 dollars on video games

He had 195-160-5=30 dollars left

He had 195-30=165 dollars from selling legos

He had 13-165/15=13-11=2 legos left

The answer is 2\\

\end{tcolorbox}
\vspace{-10pt}
  \caption{Visualization of prediction results of LLM. The reasoning errors are highlighted in \textcolor{red}{\textbf{red}} bold.}
  \label{fig:visualization2}
\end{figure*}

\begin{figure*}
  \linespread{1.2}\selectfont
  \begin{tcolorbox}
  [colback=white!95!gray, colframe=gray!50!black, rounded corners, title={Question: Charlotte went into the kitchen supply store knowing she wanted a set of pot and pans for \$120.00, a set of mixing bowls for \$20.00 and 5 separate utensils at \$5.00 each.  At check out, the clerk told her everything was 20\% off.  How much was her order?\\
  Correct: 132}]
  
  \textbf{MLE}:
  
She had a \$120.00 set of pots and pans and a \$20.00 set of mixing bowls so that's 120+20 = \$140.00

She had 5 separate utensils at \$5.00 each so that's 5*5 = \$25.00

Her order was \$\textcolor{red}{\textbf{140.00}} for pots and pans, \$20.00 for mixing bowls and \$25.00 for utensils for a total of 140+20+25 = \$185.00

Her order was \$185.00 and it was 20\% off so that's .20*185 = \$37.00

Her order was \$185.00 and it was 20\% off so she paid \$37.00 less so her order was 185-37 = \$148.00

The answer is 148\\

  \textbf{NTL}:

She had a \$20.00 mixing bowl and 5 separate utensils at \$5.00 each so that's 20+5*5 = \$\textcolor{red}{\textbf{35.00}}

Her order was \$120.00 for the pots and pans, \$20.00 for the mixing bowls and \$35.00 for the utensils for a total of 120+20+35 = \$175.00

Her order was \$175.00 and it was 20\% off so that's 175*.20 = \$35.00

Her order was \$175.00 and it was 20\% off so that's 175-35 = \$140.00

The answer is 140\\

  \textbf{DIST$^2$Loss}:

She had a \$120.00 set of pots and pans and a \$20.00 set of mixing bowls so that's 120+20 = \$140.00

She had 5 separate utensils at \$5.00 each so that's 5*5 = \$25.00

Her order was \$\textcolor{red}{\textbf{140.00}} for pots and pans, \$20.00 for mixing bowls and \$25.00 for utensils for a total of 140+20+25 = \$185.00

Her order was \$185.00 and it was 20\% off so that's 185*.20 = \$37.00

Her order was \$185.00 and it was 20\% off so that's \$185-\$37 = \$148.00

The answer is 148\\

  \textbf{DEL (Ours)}:
  
She had a set of pots and pans for \$120.00 and mixing bowls for \$20.00 so her order was \$120.00 + \$20.00 = \$140.00

She had 5 separate utensils at \$5.00 each so that was 5 * \$5.00 = \$25.00

Her order was \$140.00 + \$25.00 = \$165.00

Her order was \$165.00 and it was 20\% off so that's 20/100 * \$165.00 = \$33.00 off

Her order was \$165.00 - \$33.00 = \$132.00

The answer is 132\\

\end{tcolorbox}
\vspace{-10pt}
  \caption{Visualization of prediction results of LLM. The reasoning errors are highlighted in \textcolor{red}{\textbf{red}} bold.}
  \label{fig:visualization3}
\end{figure*}

\begin{figure*}
  \linespread{1.2}\selectfont
  \begin{tcolorbox}
  [colback=white!95!gray, colframe=gray!50!black, rounded corners, title={Question: Valerie earns \$5000 per month, 1/2 of what her brother earns. If their mother earns twice their combined salary, what's the total amount of money they all have together?\\
  Correct: 45000}]
  
  \textbf{MLE}:
  
Valerie earns \textcolor{red}{\textbf{\$5000/2 = \$2500}} per month.

Valerie's brother earns \$2500/2 = \$1250 per month.

Their combined salary is \$2500 + \$1250 = \$3750 per month.

Their mother earns 2 * \$3750 = \$7500 per month.

The total amount of money they all have together is \$5000 + \$3750 + \$7500 = \$16250.

The answer is 16250\\

  \textbf{NTL}:

Valerie earns 1/2 of what her brother earns, meaning her brother earns 5000*2=10000 per month.

Together, Valerie and her brother earn 10000+5000=15000 per month.

Their mother earns twice their combined salary, meaning she earns 2*15000=30000 per month.

The total amount of money they all have together is 30000+15000=45000 per month.

The answer is 45000\\

  \textbf{DIST$^2$Loss}:

Valerie earns \textcolor{red}{\textbf{\$5000/2 = \$2500}} per month.

Her brother earns \$2500 * 2 = \$5000 per month.

Their combined salary is \$2500 + \$5000 = \$7500 per month.

Their mother earns 2 * \$7500 = \$15000 per month.

The total amount of money they all have together is \$15000 + \$7500 = \$22500.

The answer is 22500\\

  \textbf{DEL (Ours)}:
  
Valerie's brother earns 2 * 5000 = 10000 dollars per month.

Valerie's mother earns 2 * (5000 + 10000) = 30000 dollars per month.

The total amount of money they all have together is 5000 + 10000 + 30000 = 45000 dollars.

The answer is 45000\\

\end{tcolorbox}
\vspace{-10pt}
  \caption{Visualization of prediction results of LLM. The reasoning errors are highlighted in \textcolor{red}{\textbf{red}} bold.}
  \label{fig:visualization4}
\end{figure*}

\begin{figure*}
  \linespread{1.2}\selectfont
  \begin{tcolorbox}
  [colback=white!95!gray, colframe=gray!50!black, rounded corners, title={Irrelevant context distraction}]
  
  \textbf{Question}:During summer break 644997 kids from Lawrence county stayed home and the other 893835 kids went to camp. An additional 78 kids from outside the county attended the camp. How many kids are in Lawrence county?

  \textbf{Correct}: 1538832

  \textbf{Answer}:
First add the number of kids from Lawrence county who stayed home to the number of kids who went to camp: 644997 + 893835 = 1538832

\textcolor{red}{\textbf{Then add the number of kids from outside the county who attended the camp to find the total number of kids in Lawrence county: 1538832 + 78 = 1539010}}

The answer is 1539010
\end{tcolorbox}

  \begin{tcolorbox}
  [colback=white!95!gray, colframe=gray!50!black, rounded corners, title={Calculation error}]

  \textbf{Question}:
Question: Michael rented a truck for one day. There was a rental fee of 20.99 dollars , and there also was a charge of 25 cents for each mile. Michael paid 95.74 dollars when he returned the truck. For how many miles did he drive the truck?\\
  \textbf{Correct}: 299

  \textbf{Answer}:
  Let x be the number of miles Michael drove the truck.
  
  The total cost is 20.99 + 0.25*x = 95.74
  
  Subtracting 20.99 from both sides, we get 0.25*x = 74.75
  
  Dividing both sides by 0.25, we get x = \textcolor{red}{\textbf{303}}
  
  The answer is 303
\end{tcolorbox}
\vspace{-10pt}
  \caption{Failure cases. The reasoning errors are highlighted in \textcolor{red}{\textbf{red}} bold.}
  \label{fig:failure-cases}
\end{figure*}

\end{document}